\documentclass[final,12pt]{colt2021} % Final submission

\pdfoutput=1

\usepackage{times}
\title[Information-Theoretic Generalization Bounds for SGD]{Information-Theoretic Generalization Bounds
\\for Stochastic Gradient Descent}
\usepackage{times}
\usepackage[utf8]{inputenc} % allow utf-8 input
\usepackage[T1]{fontenc}    % use 8-bit T1 fonts
\usepackage{hyperref}       % hyperlinks
\usepackage{url}            % simple URL typesetting
\usepackage{booktabs}       % professional-quality tables
\usepackage{amsfonts}       % blackboard math symbols
\usepackage{nicefrac}       % compact symbols for 1/2, etc.
\usepackage{microtype}      % microtypography
\usepackage{xspace}
\usepackage{amsmath}
\usepackage{bm}

\usepackage{mathtools}  
\usepackage{amsmath}
\usepackage{amssymb}
\usepackage{tabulary}
\usepackage{booktabs}

\usepackage{algorithm}
\usepackage{algpseudocode}
\usepackage{algpascal}
\usepackage{comment}
\usepackage{selectp}
\usepackage{selectp}

\newcommand{\alg}{\mathcal{A}}

\newcommand{\bg}{\overline{g}}

\newcommand{\dd}{\mathrm{d}}

  \newcommand{\Zw}{\mathcal{Z}}

\newcommand{\N}{\mathcal{N}}
\newcommand{\real}{\mathbb{R}}

\newcommand{\Sw}{\mathcal{S}}

\newcommand{\DD}[2]{\mathcal{D}\pa{#1\middle\|#2}}

\newcommand{\EE}[1]{\mathbb{E}\left[#1\right]}
\newcommand{\bEE}[1]{\mathbb{E}\bigl[#1\bigr]}

\newcommand{\EEs}[2]{\mathbb{E}_{#2}\left[#1\right]}

\newcommand{\EEcc}[2]{\mathbb{E}\left[\left.#1\right|#2\right]}

\newcommand{\iprod}[2]{\left\langle#1,#2\right\rangle}

\newcommand{\norm}[1]{\left\|#1\right\|}

\newcommand{\twonorm}[1]{\norm{#1}}

\newcommand{\ev}[1]{\left\{#1\right\}}
\newcommand{\pa}[1]{\left(#1\right)}

\newcommand{\wt}{\widetilde}

\newcommand{\bsigma}{\bm{\sigma}}
\newcommand{\bSigma}{\bm{\Sigma}}

\newcommand{\tQ}{\wt{Q}}
\newcommand{\tW}{\wt{W}}

\newcommand{\tw}{\wt{w}}

\newcommand{\gen}{\textup{gen}}

\newcommand{\qed}{\hfill\BlackBox\\[2mm]}

\coltauthor{\Name[{Gergely~Neu}]{Gergely Neu} \Email{gergely.neu@gmail.com}\\
 \addr Universitat Pompeu Fabra, Barcelona, Spain
 \AND
 \Name[{Gintare~Karolina~Dziugaite}]{Gintare Karolina Dziugaite} \Email{karolina.dziugaite@elementai.com}\\
 \addr {Element AI / Mila, Montreal, Canada}
 \AND
\Name[{Mahdi~Haghifam}]{Mahdi Haghifam} \Email{mahdi.haghifam@mail.utoronto.ca}\\
 \addr {University of Toronto / Vector Institute, Toronto, Canada}
 \AND
 \Name[{Daniel~M.~Roy}]{Daniel M. Roy} \Email{daniel.roy@utoronto.ca}\\
 \addr {University of Toronto / Vector Institute, Toronto, Canada}
 }

\begin{document}

\maketitle

\begin{abstract}
We study the generalization properties of the popular stochastic optimization method known as stochastic gradient 
descent (SGD) for optimizing general non-convex loss functions. Our main contribution is providing upper bounds on the 
generalization error that depend on local statistics of the stochastic gradients evaluated along the path of iterates 
calculated by SGD. The key factors our bounds depend on are the variance of the gradients (with respect to the data 
distribution) and the local smoothness of the objective function along the SGD path, and the sensitivity of the loss 
function to perturbations to the final output. Our key technical tool is combining the information-theoretic 
generalization bounds previously used for analyzing randomized variants of SGD with a perturbation analysis of the 
iterates.
\end{abstract}


\section{Introduction}\label{sec:intro}
Stochastic gradient descent (SGD) is arguably the single most important algorithmic component of the modern 
machine-learning toolbox. First proposed by \citet{RM51} for finding roots of a function using noisy evaluations, 
stochastic approximation methods like SGD has been broadly adapted for a variety of tasks in signal processing, control 
and optimization \citep{KY97,NJLS09}. In the 
context of machine learning, SGD is extremely popular due to its efficient use of computation that naturally enables it 
to process very large data sets and its inherent ability to handle noisy data \citep{BoBo07}. In modern 
machine learning, SGD is the de facto standard method for training deep neural networks, partly due to 
the efficient computability of the gradients through the famous backpropagation algorithm 
\citep{rumelhart86,lecun2012efficient}. Given the surprising effectiveness of SGD for deep learning, several empirical 
and theoretical studies have attempted to explain the reasons of its success. In recent years, this research 
effort has lead to some remarkable results that finally shed some light on the core factors contributing to the 
effectiveness of SGD. A handful of examples include showing guaranteed convergence to minimizers \citep{GHJY15,LSJR16}, 
to minimum-norm solutions under various assumptions \citep{MBB18,JGH18,OS19}, and even to global optima of 
overparametrized neural networks \citep{DZPS18,ALS19}. 

This paper aims to contribute to a better understanding of the \emph{generalization properties of SGD} for optimizing 
general non-convex objectives.
This is a widely studied question, mostly from the perspective of \emph{stability}, inspired by the understanding that 
stable learning algorithms are guaranteed to generalize well to test data 
\citep{bousquet2002stability,feldman2019high,bousquet2020sharper}. This line of work was initiated by \citet*{HRS16}, 
who showed that SGD has strong stability properties when applied to smooth convex loss functions, and is particularly 
stable when the objective is also strongly convex. They also provided bounds on the generalization error for non-convex 
losses, with a rate that is polynomial in the number of steps $T$, with an exponent that depends on the smoothness of 
the objective. A significant limitation of their results is that they require the loss function to have a bounded 
Lipschitz constant, which is generally difficult to ensure (especially when training deep neural networks). In recent 
years, the results of \citeauthor{HRS16} have been strengthened in a variety of ways, for example by removing the 
Lipschitz condition in the general non-convex case by \citet{LY20} and by removing the smoothness condition in the 
convex case by \citet{BFGT20}.

Another line of work studies the generalization properties of a randomized version of SGD called 
\emph{stochastic gradient Langevin dynamics} (SGLD, cf.~\citealp{WT11}). As shown by \citet{RRT17}, SGLD has strong 
finite-sample convergence properties for general non-convex learning, which already implies good generalization.
Another line of attack aiming to directly understand the generalization of SGLD was initiated by \citet*{PJL18}. 
Subsequent  improvements to this technique were made by \citet{NHDKR19}, \citet{HNKRD20} and \citet{RBTS21}, who prove 
data-dependent generalization bounds that do not depend on the Lipschitz constant of the loss function. The core 
technical tool of these analyses is bounding the generalization error in terms of the mutual information between inputs 
and outputs of learning algorithms, previously proposed in a much more general context by \citet{RZ16,RZ19} and 
\citet{XR17}, which are themselves closely connected to classic PAC-Bayesian generalization error bounds 
\citep{McA99,McA13}. Indeed, these information-theoretic tools are particularly suitable for analyzing SGD-like 
algorithms due to a convenient decomposition of the mutual infomation across iterations by an application of the chain 
rule of the relative  entropy. The obvious downside of this technique is that it relies on randomly perturbing the SGD 
iterates, which empirically hurts the performance of the algorithm and results in underfitting the training loss.

Our main contribution in this paper is demonstrating that it is possible to conduct an information-theoretic analysis 
for the vanilla version of SGD, without directly perturbing the iterates. Our key observation is that it is sufficient 
to add carefully constructed random noise to the iterates \emph{only during the analysis}, which then allows using the 
technique of \citet{PJL18} on these ``virtual SGLD'' iterates. The key challenge in the analysis is designing 
the perturbations and correctly handling their propagation through the iterations. The bounds we derive depend on three 
factors: the variance of the gradient updates and the sensitivity of the gradients to perturbations along the path, and 
the sensitivity of the final output to perturbations. All of these quantities are evaluated locally along the 
iterates that SGD produces. Another important consequence of our virtual perturbation technique is that our bounds hold 
simultaneously for all noise distributions, which obviates the need to tune the hyperparameters of the algorithm to 
optimize the bound.

As is probably obvious from the above discussion, our approach is thoroughly inspired by the insightful work of 
\citet*{PJL18}, which itself is inspired by the seminal works of \citet{RZ16,RZ19} and \citet{XR17} on 
information-theoretic generalization bounds. In recent years, their theory has been extended in a variety of directions 
that resulted in tighter and tighter bounds. Here, we highlight the works of \citet{AAV18,AA20} whose chaining 
technique 
can lead to tighter bounds when applied to neural networks, and the work of \citet{SZ20} whose key idea of 
appropriately 
conditioning the mutual information enables proving more refined data-dependent guarantees. In fact, these latter ideas 
directly inspired the work of \citet{HNKRD20} and \citet{RBTS21} on tighter generalization bounds for Langevin dynamics 
and SGLD, respectively. We conjecture that our analysis can be also refined by using such sophisticated 
information-theoretic techniques. Finally, we also mention that our 
main idea of using random perturbations to guarantee boundedness of the mutual information has been partially inspired 
by the work of \citet{ZHBHB20}.

Besides the works mentioned above, several other theories of generalization have been proposed in the deep learning 
literature. One particularly widely held belief is that algorithms that find ``wide optima'' of the loss landscape 
generalize well to test data. This hypothesis has been first proposed by \citet{HS97} and later popularized by 
\citet{KNTMS17}, and has attracted some (moderately rigorous) verification and refutation attempts (e.g., 
\citealp{DPBB17,IPGVW18,HHY19,chaudhari2019entropy}). One common criticism of this theory is that the ``width'' of a 
solution is difficult to formally define, and the most common intuitive definitions suffer from being sensitive 
to reparametrization. While we don't claim to resolve the debate around ``wide optima'', our 
results lend some minimal credence to this theory in that our generalization bounds indeed predict an improvement when 
the loss of the final solution is insensitive to perturbations. As we will show, our bounds allow measuring this 
sensitivity in terms of arbitrary coordinate systems, which at least addresses the most elementary concerns with 
parametrization-sensitivity of common definitions. That said, some aspects of our results defy common wisdom: most 
notably, our bounds show an \emph{improvement} for larger batch sizes, even though this choice empirically leads to 
worse performance, purportedly due to converging to ``sharper minima''. Importantly, our analysis does \emph{not} 
explain why SGD would converge towards solutions that generalize well and how hyperparameters like the batch size could 
impact the quality of solutions.

\paragraph{Notation.} For two distributions $P$ and $Q$ satisfying $P\ll Q$, we denote their relative entropy by 
$\DD{P}{Q} = \int_x \dd P(x) \log \pa{\frac{\dd P}{\dd Q}(x)}$. The distribution of a random variable $X$ will be 
denoted by $P_X$ 
and the product distribution between $P_X$ and $P_Y$ will be denoted by $P_X\otimes P_Y$. With this notation, the 
mutual information between two random variables is defined as $I(X;Y) = \DD{P_{X,Y}}{P_X\otimes P_Y}$.
Whenever possible, we use capital letters to denote random variables and use lowercase letters to denote their 
realizations. We use $\twonorm{\cdot}$ to denote the Euclidean norm on $\real^d$ 
and $\Sw_+$ to denote the set of symmetric positive definite matrices in $\real^{d\times d}$. For any $u\in\real^d$ and 
$\Sigma\in\Sw_+$, we will use $\N(u,\Sigma)$ to denote the multivariate normal distribution with mean $u$ and 
covariance $\Sigma$.

\section{Background}
We let $S=\ev{Z_1,Z_2,\dots,Z_n}$ denote a data set of $n$ i.i.d.~samples taking value in the set $\Zw$ and $W = 
\alg(S)$ be the output of a (potentially randomized) learning algorithm run on data set $S$. We assume that this output 
is a $d$-dimensional real-valued vector, representing parameters of a potentially nonlinear model such as a neural 
network. The performance of a learning algorithm is evaluated in terms of a loss function mapping data points and 
parameter vectors to positive real numbers, with $\ell(w,z)$ giving the loss of the model with parameters $w\in\real^d$ 
evaluated on data point $z\in\Zw$. 
We will assume throughout that $\ell(w,z)$ is differentiable with respect to $w$ for all $z$, and let $g(w,z)$ 
denote its gradient evaluated at $w$. Furthermore, we will assume that the distribution of $Z\sim Z_1$ is such that 
$\ell(w,Z)$ is $R$-subgaussian for any $w\in\real^d$ in the sense that the inequality $\EE{\exp(y\pa{\ell(w,Z) - 
\EE{\ell(w,Z)}})} \le \exp(R^2 y^2/2)$ is satisfied for all $y\in\real$. This condition clearly holds if the 
loss function is bounded on the support of the data distribution. 
Letting $S'=\ev{Z_1',Z_2',\dots,Z_n'}$ be an independent data set of the same distribution as $S$, and denoting the 
average loss of $w$ on data set $s=\ev{z_1,z_2,\dots,z_n}$ by $L(w,s) = \frac 1n \sum_{i=1}^n \ell(w,z_i)$, we define 
our main object of interest, the \emph{expected generalization error} of algorithm $\alg$ on the data set $S$ as
\[
 \gen(W,S) = \EE{L(W,S') - L(W,S)}.
\]
The expected generalization error (that we will often simply call \emph{generalization error}) measures the difference 
between the training loss $L(W,S)$ and the test loss $\EEcc{L(W,S')}{W}$ on expectation with respect to the randomness 
of the data set $S$ and the output $W$.
Our techniques will be based on the now-classic results of \citet{RZ16,XR17} that show that, under the 
conditions stated above, the generalization error of any learning algorithm can be bounded as
\begin{equation}\label{eq:gen}
 \left|\gen(W,S)\right| \le \sqrt{\frac{2 R^2 I(W;S)}{n}},
\end{equation}
where $I(W;S)$ denotes the mutual information between the random variables $W$ and $S$.

We will consider the classic stochastic gradient descent (SGD) algorithm originally proposed by \citet{RM51} and 
applied here to approximately minimize the empirical loss $L(w,S)$ in terms of $w$ on the data set $S$. This iterative 
algorithm operates by making sequential updates to a parameter vector in the direction opposite to the gradient of the 
loss function evaluated at a minibatch of data points. Such a minibatch estimator will be denoted with a slight abuse 
of 
notation as $g(w,Z_J)= \frac {1}{|J|} \sum_{i\in J} g(w, Z_i)$, where $J\subseteq[n]$ is a set of indices.
Then, the iterates of SGD are given by drawing $W_1$ from an arbitrary fixed distribution independent of $S$, and then 
updating the parameters recursively as
\begin{equation}\label{eq:sgd}
 W_{t+1} = W_t - \eta_t G_t = W_t - \eta_t g(W_t, Z_{J_t})
\end{equation}
for all $t=1,\dots,T$, where $\eta_t$ is a positive learning-rate parameter and $J_t\subseteq[n]$ is the set of $b_t$ 
indices of the minibatch selected for the $t$-th update, and $G_t = g(W_t,Z_{J_t})$. We will assume that $\eta_t$ and 
$J_t$ are chosen independently of the history or the data set, but otherwise make no restrictions about them (e.g., we 
allow increasing and cyclic stepsizes, and randomized minibatch schedules such as random shuffling). We 
will often denote the minibatch $Z_{J_t}$ by $B_t$ for brevity, and we will sometimes refer to the quantity 
$\sum_{t=1}^T b_t / n$ as the number of passes over the data set.

\section{Generalization-error bounds for SGD}
Our main result is a bound on the generalization error of the final iterate $W_T$ produced by SGD as defined in the 
previous section. The key quantities appearing in the bound are the following:
\begin{itemize}
 \item The \emph{local value sensitivity} of the loss function at $w\in\real^d$, defined on data set $s\in\Zw^n$ at 
level $\sigma$ as
 \[
  \Delta_\sigma(w,s) = \EE{L(w,s) - L(w + \xi,s)}
 \]
 where $\xi \sim \N(0,\sigma^2I)$. When the loss is $\mu$-smooth, this quantity is bounded by 
$\mu\sigma^2d/2$.
 \item The \emph{local gradient sensitivity} of the loss function, defined at $w\in\real^d$ and at level $\sigma$ 
as
 \[
  \Gamma_\sigma(w) = \EE{\twonorm{\bg(w) - \bg(w + \xi)}^2},
 \]
where $\xi\sim\N(0,\sigma^2I)$, and $\bg(w) = \EE{g(w,Z)}$ is the population gradient evaluated at $w$. This 
quantity characterizes the sensitivity of the gradients of the expected loss function to perturbations around $w$, and 
is upper-bounded by $\mu^2 \sigma^2d$ when the loss function is globally $\mu$-smooth.
 \item The \emph{local gradient variance} of the loss function defined at $w\in\real^d$ as
 \[
  V_t(w) = \EEcc{\twonorm{g(w,B_t) - \bg(w)}^2}{W_t=w}.
 \]
 It is important to note that this is a non-standard notion of variance in that it measures the expected squared 
Euclidean distance from the \emph{population gradient} $\EE{g(w,Z)}$ instead of the conditional expectation of the 
gradient $\EEcc{g(w,Z)}{W_t=w}$.
\end{itemize}
Our main result is stated as follows:
\begin{theorem}\label{thm:main}
Fix any sequence $\bsigma = \pa{\sigma_1,\sigma_2,\dots,\sigma_T}$ of positive real numbers and define $\sigma_{1:t} = 
\sqrt{\sum_{k=1}^{t-1} \sigma_k^2}$ for all $t$. Then, the generalization error of the final iterate of SGD satisfies
\[
|\gen(W_T,S)| \le \sqrt{\frac{4R^2}{n} \sum_{t=1}^T \frac{\eta^2_t}{\sigma_{t}^2} 
\EE{{\Gamma_{\sigma_{1:t}}(W_t) + V_t(W_t)}}} + \Bigl|\bEE{\Delta_{\sigma_{1:T}}(W_T,S') - 
\Delta_{\sigma_{1:T}}(W_T,S)}\Bigr|.
\]
\end{theorem}
The proof is provided in Section~\ref{sec:analysis}. There are several interesting properties of this result. 
One of its most important merits is that it depends on \emph{pathwise} statistics of the SGD iterates, as the 
local gradient sensitivity and variance parameters are evaluated along the path taken by SGD.
Consequently, the bound becomes small whenever the gradients demonstrate little variability along the path, both in 
terms of varying $w$ and $Z$. 
% The 
% impact of the minibatch size is clearly captured by the bound, showing that large batches result in smaller gradient 
% variance and an improved generalization guarantee. 
Finally, the bound also depends on the sensitivity of the loss 
function to perturbations around the final output $W_T$, measured both on the training set $S$ and the test set $S'$. 
Intuitively, this term becomes small if SGD outputs a parameter vector in a flat area of the loss surface.

While it may feel unsatisfying that the tradeoffs involving $\bsigma$ are not characterized explicitly, one can find 
consolation in the remarkable fact that the bound simultaneously holds for all possible choices of $\bsigma$. 
Indeed, this is a very useful property given that the bound presents some delicate tradeoffs involving the parameters 
$\bsigma$: large values of $\sigma_t^2$ decrease the first term in the bound related to the variability of the 
gradients, at the expense of increasing the second term related to the flatness of the loss around the final iterate. 
More generally, the optimal choice of these parameters can depend on the local properties of the loss functions 
along the expected SGD path. One important limitation is that the bound only holds for \emph{fixed} 
sequences, and in particular that $\sigma_t$ is not allowed to depend on the past iterates $W_{1:t}$. 
Similarly, the 
theorem crucially requires the sequence of learning rate schedule to be fixed independently of the data. We provide a 
more detailed discussion of this issue at the end of Section~\ref{sec:variations}.
On the other hand, it is possible to further extend the flexibility of the bound by allowing more general perturbation 
covariances; details are provided in Section~\ref{sec:invariant}.

% 
% Similarly, the 
% theorem crucially requires the sequence of learning rate schedule to be fixed independently of the data. We provide a 
% more detailed discussion of this issue at the end of Section~\ref{sec:variations}.

\subsection{Generalization-error guarantees for smooth loss functions} 
To provide some intuition about the magnitude of the terms in the bound, we state the following corollary that 
provides a simpler bound under some concrete assumptions on the loss function and the parameters:
\begin{corollary}\label{cor:smooth}
Suppose that $\eta_t = \eta$ and $b_t = b$ for all $t$ and the minibatches are chosen so that for each $i\in[n]$, there 
is exactly one index $t$ such that $i\in J_t$. Furthermore, suppose that $\bEE{\twonorm{g(w,Z) - \bg(w)}^2} \le v$ for 
all $w$ and that $\ell$ is globally $\mu$-smooth in the sense that the inequality $\twonorm{g(w,z) - g(w+u,z)} \le \mu 
\twonorm{u}$ holds for all $w,u\in\real^d$ and all $z\in\Zw$.
Then,  the generalization error of the final iterate of SGD satisfies the following bound for any $\sigma$:
\[
\left|\gen(W_T,S)\right| = O\pa{\sqrt{\frac{R^2 \eta^2 T}{n} \pa{\mu^2 d T + 
\frac{v}{b\sigma^2}}} + \mu \sigma^2 d T}.
\] 
\end{corollary}
The proof follows from noticing that 
\[
  \Gamma_{\sigma_{1:t}}(w) = \EE{\norm{g(w+\xi_t,z) - g(w,z)}^2}\le \mu^2 \bEE{\norm{\xi_t}^2} = \mu^2 
d \sigma_{1:t}^2 = \mu^2 d \sigma^2 t
 \]
 and that $\Delta_{\sigma_{1:T}}(W_T,S)$ and $\Delta_{\sigma_{1:T}}(W_T,S')$ can be bounded using
\[
  \left|\EE{\ell(W_T,z) - \ell(W_T+\xi_T,z)}\right| \le \left|\EE{\iprod{\nabla \ell(w,z)}{\xi_T}}\right| + 
\frac{\mu}{2} \bEE{\norm{\xi_T}^2} = \frac{\mu \sigma_{1:T}^2 d}{2} = \frac{\mu \sigma^2 d T}{2},
 \]
where the equality follows from using $\EE{\xi_T} = 0$ and the independence of $W_T$ and $\xi_T$. Furthermore, 
the independence of $Z_{J_t}$ and $W_t$ implies that $V_t(W_t) = \frac 1b \mathbb{E}\bigl[\twonorm{g(W_t,Z) - 
\EE{g(W_t,Z')}}^2\bigr] \le \frac vb$.

The rates become better as $T$ and $\eta$ are decreased, but doing so can hurt the fit on the training data. 
Furthermore, the rates also improve as $b$ is increased, although only until a critical value where the term $\mu^2 dT$ 
becomes dominant. Such tradeoffs involving the batch size are not uncommon in the related literature 
(see, e.g.,~\citealp{lin2020extrapolation}). The noise variance $\sigma^2$ still influences the bound in a relatively 
complex way, but should be tuned as a function of the minibatch gradient variance $v/b$: as this quantity approaches 
zero, one can afford to set smaller perturbations and improve the last sensitivity term in the bound. Once again, we 
highlight that the optimal tuning of $\sigma$ does not require prior knowledge of problem parameters like $v$ and 
$\mu$. 

We discuss the rates that can be derived from the bound in some important settings:
\begin{description}
 \item[Small-batch SGD:] Setting $T = O(n)$ and $b=O(1)$, it is not possible derive a bound that vanishes with large 
$n$ under the classic stepsize choice $\eta = O(1/\sqrt{n})$. That said, using $\eta = O(1/n)$ and $\sigma = 
\Theta(n^{-4/3})$, it is possible to guarantee the vanishing rate $|\gen(W;S)| = O(n^{-1/3})$.
%  \item[Multi-pass batch GD:] Setting $b = n$ yields batch gradient descent. In this case, one can set the the number 
of 
% passes as $T = O(\sqrt{n})$ and the stepsize as $\eta = O(1/T) = O(1/\sqrt{n})$, and $\sigma = \Theta(1/\sqrt{n})$ to 
% obtain a rate of $|\gen(W;S)| = O(1/\sqrt{n})$.
 \item[Large-batch SGD:] Setting $T = O(\sqrt{n})$ and $b = 
\Omega(\sqrt{n})$, the stepsize as $\eta = O(1/T) = O(1/\sqrt{n})$, and $\sigma = \Theta(1/\sqrt{n})$ we obtain a rate 
of $|\gen(W;S)| = O(1/\sqrt{n})$.
\end{description}
Interestingly, the rates for the latter case saturate when setting $b=\Theta(\sqrt{n})$ and do not improve any further 
even when setting $b = \omega(\sqrt{n})$. Thus, even if the bounds of Corollary~\ref{cor:smooth} remain qualitatively 
true for larger batch sizes, no further improvement is to be expected.
In both cases discussed above, the rates are obviously far from being tight in general, especially in the small-batch 
case where using stepsizes of order $1/\sqrt{n}$ are known to lead to bounds of order $1/\sqrt{n}$ for general convex 
functions, and even better rates are known when a smoothness assumption is also in place. That said, the above examples 
show that it is indeed possible to achieve generalization-error bounds that are vanishing in $n$, for parameter 
settings that are not entirely unrealistic. Notably, unlike the rates proved by \citet{HRS16}, the exponent of the 
rates 
we can guarantee is independent of the smoothness parameter and the rates themselves are independent of the Lipschitz 
constant of the loss function. One downside of our guarantees is their direct dependence on the dimension $d$ which can 
be attenuated by using more general perturbation distributions that are better adapted to the geometry of the loss 
function (cf.~Section~\ref{sec:invariant}).

\subsection{Comparison with SGLD}\label{sec:SGLD}
To put our result into perspective, we now also describe the Stochastic Gradient Langevin Dynamics (SGLD) algorithm 
that is essentially a variant of SGD that adds isotropic Gaussian noise to its iterates \citep{GM91,WT11}. 
Specifically, the iterates of SGLD are given by the recursion
\begin{equation}\label{eq:sgld}
 W_{t+1} = W_t - \eta_t g(W_t,Z_{J_t}) + \varepsilon_t,
\end{equation}
for all $t$, where $\varepsilon_t \sim \N(0,\sigma_t^2 I)$, and the hyperparameters $\eta_t$, $\sigma_t$ and $J_t$ are 
chosen according to an arbitrary rule oblivious to the data set $S$. As first shown by \citet{PJL18}, the 
generalization 
error of this algorithm can be directly bounded in terms of the mutual information between $W_T$ and $S$, which itself 
can be shown to be of order $C^2\sum_{t=1}^T \eta_t^2/\sigma_t^2$, under the assumption that the loss function is 
$C$-Lipschitz. This bound has been improved by \citet{NHDKR19} and \citet{HNKRD20}, who replace the Lipschitz constant 
by a data-dependent quantity that is often orders of magnitude smaller. Instead of reproducing their rather involved 
definitions, we state the following simple guarantee that can be obtained via a straightforward modification of the 
proof of our Theorem~\ref{thm:main}:
\begin{proposition}\label{prop:SGLD}
The generalization error of the final iterate of SGLD satisfies
\[
\gen(W_T,S) \le \sqrt{\frac{R^2}{n} \sum_{t=1}^T \frac{\eta^2_t}{\sigma_{t}^2}\EE{V_t(W_t)}}.
\] 
\end{proposition}
While this bound can be weaker than the data-dependent ones mentioned above, they are quite directly comparable to our 
results concerning SGD.
One key difference is that the sensitivity terms disappear from the generalization bound, which can be 
attributed to SGLD being inherently more stable than SGD. This improved generalization-error guarantee however comes at 
the price of a worse training error, owing to the presence of the perturbations in the updates. While this 
degradation is difficult to quantify in general, it is qualitatively related to the effect captured by the sensitivity 
terms in our bound for SGD: intuitively, large sensitivity of the gradients to perturbations along the path can 
negatively impact the convergence speed in convex settings, and can result in large variations of 
the final iterate in nonconvex settings. Thus, in a certain sense, the sensitivity terms in our bound for SGD are 
pushed into the training error of SGLD. That said, there are known cases (e.g., strongly convex loss functions) where 
adding perturbations to the iterates incrementally is known to result in significantly better excess-risk guarantees 
than perturbing the final iterate \citep{feldman2018privacy}. 
The extent to which this holds for general non-convex functions is unclear.

A major downside of SGLD is that it requires prior commitment to the sequence of perturbation parameters $\bsigma$. 
This is made complicated by the poorly-understood tradeoffs between the generalization error and the training error: 
while the former is improved by setting large perturbation variances, the latter is clearly hurt by it. Since this 
tradeoff is 
characterized even less explicitly than in our main result regarding SGD, it is virtually impossible to set $\bsigma$ 
in 
a way that optimizes both terms of the excess risk. Thus, even if the sensitivity terms in our bound are worse than the 
excess training loss of SGLD for a fixed $\bsigma$, our bounds still have the major advantage of simultaneously holding 
for all noise distributions, thus obviating the need to tune hyperparameters of the algorithm itself. Based on the 
above discussion, we find it plausible that SGLD may generally be able to achieve better excess risk than SGD, 
although with the major caveat that tuning its hyperparameters is prohibitively complex as compared to SGD.

\section{Analysis}\label{sec:analysis}
The core idea of our analysis is applying the generalization bound~\eqref{eq:gen} to a perturbed version of the output 
$W_T$, making sure that the mutual information between the input and the perturbed output is bounded. To be precise, we 
define a perturbed version of the output as $\tW_T = W_T + \xi_T$, where $\xi_T$ is a random perturbation independent 
from 
the data and $W_T$. For the proof of our main theorem, we will use perturbations from a zero-mean isotropic Gaussian 
distribution with variance $\sigma_{1:T}^2$, that is, $\xi_T\sim\N(0,\sigma_{1:T}^2I)$. 
It is then straightforward to apply the generalization-error guarantee~\eqref{eq:gen} to the pair 
$(\tW_T,S)$ and bound the generalization error as
\begin{equation}\label{eq:sensitivity}
 \begin{split}
 \gen(W_T,S) &= \gen(\tW_T,S) + \EE{L(W_T,S') - L(\tW_T,S')} + \EE{L(\tW_T,S) - L(W_T,S)}.
\\
&\le \sqrt{\frac{2R^2 I(\tW_T;S)}{n}} + \EE{\Delta_{\sigma_{1:T}}(W_T,S') - \Delta_{\sigma_{1:T}}(W_T,S)}.
 \end{split}
\end{equation}
The key challenge in the proof is controlling the mutual information $I(\tW_T;S)$. In order to bound this term, our 
analysis makes direct use of techniques first introduced by \citet{PJL18}, with subtle modifications made to account 
for 
the fact that our perturbations do not appear as part of the algorithm, but are only defined to aid the analysis.

Our main technical idea is constructing the perturbation $\xi_T$ in an incremental manner, and using these 
incremental perturbations to define a \emph{perturbed SGD path}. In particular, we define 
the perturbations $\varepsilon_t\sim\N(0,\sigma_t^2I)$ and the perturbed SGD iterates through the recursion
\begin{equation}\label{eq:surrogate}
 \tW_{t+1} = \tW_t - \eta_t G_t + \varepsilon_t = \tW_t - \eta_t g(W_t,Z_{J_t}) + \varepsilon_t.
\end{equation}
This procedure can be seen to produce $\tW_t = W_t + \xi_t$ with $\xi_t = \sum_{k=1}^{t-1} \varepsilon_k \sim 
\N(0,\sigma_{1:t}^2)$, eventually yielding $\tW_T = W_T + \xi_T$ as output.
Intuitively, these iterates can be thought of as being between the SGD iterates~\eqref{eq:sgd} and the SGLD 
iterates~\eqref{eq:sgld} in that they add random perturbations to each update, but the gradients themselves are 
evaluated at the unperturbed SGD iterates. This results in a solution path that is strongly coupled with the SGD 
iterates, yet is still amenable to analysis techniques introduced by \citet{PJL18} due to the presence of the 
perturbations. 

For the analysis, it will be also useful 
to define the ``ghost SGD'' iterates and their perturbed counterpart as 
\[
 W'_{t+1} = W'_t - \eta_t G'_t \qquad\mbox{and}\qquad \tW'_{t+1} = \tW'_t - \eta_t G'_t 
+ \varepsilon_t'
\]
that use the independently drawn data set $S'$ and minibatch $B_t'$ indexed by $J_t'$  to construct the gradient 
estimates $G_t' = g(W'_t,B_t')$, and the independent perturbation $\varepsilon_t'\sim\N(0,\sigma_t^2I)$. As we will 
see, bounding the mutual information between $\tW_T$ and $S$ can be reduced to bounding the relative entropy between 
the conditional distribution of $\tW_T$ given $S$ and the marginal distribution of $\tW'_T$ (which matches the 
marginal distribution of $W_T$ by definition).

The analysis crucially relies on the following general lemma that quantifies the effect of random perturbations on the 
relative entropy of random variables:
\begin{lemma}\label{lem:spreadbound}
Let $X$ and $Y$ be random variables taking values in $\real^d$ with bounded second moments and let 
$\sigma > 0$. Letting $\varepsilon\sim\N(0,\sigma^2I)$ be independent of $X$ and $Y$, the relative entropy 
between the distributions of $X + \varepsilon$ and $Y + \varepsilon$ is bounded as
\[
 \DD{P_{X+\varepsilon}}{P_{Y+\varepsilon}} \le \frac{1}{2\sigma^2} \EE{\twonorm{X - Y}^2}.
\]
\end{lemma}
The bound is tight in the sense that it holds with equality if $X$ and $Y$ are constants, although 
we also note that it can be arbitrarily loose in some cases (e.g., the left-hand side is obviously zero when $X$ and 
$Y$ are identically distributed, whereas the right-hand side can be positive if they are independent). This looseness 
can be addressed by observing that the bound allows for arbitrary 
dependence between $X$ and $Y$, so one can pick the coupling between these random variables that minimizes the bound. 
In other words, the term on the right-hand side can be replaced by the squared 2-Wasserstein distance. Essentially the 
same result has been previously shown as Lemma~3.4.2 by \citet{RS18} and  a similar 
(although much less general) connection between the relative entropy and the Wasserstein distance has been 
made by \citet{ZHBHB20}. We provide the straightforward proof of Lemma~\ref{lem:spreadbound} in 
Appendix~\ref{sec:spreadbound}.

We are now in position to state and prove our most important technical result that, together with the 
inequality~\eqref{eq:sensitivity}, will immediately imply the statement of Theorem~\ref{thm:main}:
\begin{theorem}\label{thm:mutualbound}
Fix any sequence $\bsigma = \pa{\sigma_1,\sigma_2,\dots,\sigma_T}$ of positive real numbers and let $\sigma_{1:t} 
= 
\sqrt{\sum_{k=1}^{t-1} \sigma_k^2}$ for all $t$. Then, the mutual information between $\tW_T$ and $S$ satisfies
\[
 I(\tW_T;S) \le \sum_{t=1}^T \frac{2\eta^2_t}{\sigma_{t}^2} \EE{{\Gamma_{\sigma_{1:t}}(W_t) + 
V_t(W_t)}}.
\]
\end{theorem}
\begin{proof}
We start by defining some useful notation. For all $t$, we let $Q_{t|X}$ and $\tQ_{t|X}$ respectively denote the 
distributions of $W_t$ and $\tW_t$ conditioned on the random variable $X$, and $\tQ_{1:T|X}$ denote the joint 
distribution of $\tW_{1:T} = (\tW_1,\dots,\tW_T)$ given $X$. Further, we will denote the distribution of the data set 
as $\nu$. 
Then, the mutual information between $\tW_T$ and $S$ can be bounded as
\begin{align*}
 I(\tW_T;S) &= \EE{\DD{\tQ_{T|S}}{\tQ_{T}}}
 \le \EE{\DD{\tQ_{1:T|S}}{\tQ_{1:T}}}
 = \sum_{t=1}^T \EE{\DD{\tQ_{t|\tW_{1:t-1},S}}{\tQ_{t|\tW_{1:t-1}}}},
\end{align*}
where the inequality follows from the data-processing inequality, and the last step 
uses the chain rule of the relative entropy. Using the notation introduced above, we rewrite each term as
\begin{align}
 \EE{\DD{\tQ_{t+1|\tW_{1:t},S}}{\tQ_{t+1|\tW_{1:t}}}} &= \int_{s} \int_{\tw_{1:t}} 
\DD{\tQ_{t+1|\tw_{1:t},s}}{\tQ_{t+1|\tw_{1:t}}} \dd \tQ_{1:t|s}(\tw_{1:t})  \dd \nu(s)\nonumber
\\
&= \int_{s} \int_{\tw_{1:t}} \Psi_t(\tw_{1:t},s) \dd \tQ_{1:t|s}(\tw_{1:t})  \dd \nu(s),\label{eq:psibound}
\end{align}
where we defined $\Psi_t(\tw_{1:t},s)$ as the relative entropy between $\tW_{t+1}$ and $\tW_{t+1}'$ 
conditioned on the previous perturbed iterates $\tW_{1:t} = \tW_{1:t}' = \tw_{1:t}$ and the data set $S=s$. It remains 
to bound these terms for all $t$.

To proceed, notice that under these conditions, the updates can be written as
\[
 \tW_{t+1} = \tw_t - \eta_t G_t + \varepsilon_t  \qquad\mbox{and}\qquad 
\tW_{t+1}' = \tw_t - \eta_t G_t' + \varepsilon_t'.
\] 
Thus, given the condition $\tW_{1:t} = \tW_{1:t}' = \tw_{1:t}$, the relative entropy between 
$\tW_{t+1}|S$ and $\tW_{t+1}'$ equals the relative entropy between $\eta_t G_t - \varepsilon_t|S,\tw_{1:t}$ and 
$\eta_t G_t' - \varepsilon_t'|\tw_{1:t}$. 
Furthermore, under the same condition, we have $G_t = g(W_t, B_t) = g(\tw_t - \xi_t, B_t)$ and $G_t' = g(W_t',B_t') = 
g(\tw_t - \xi_t',B_t')$, so we can appeal to Lemma~\ref{lem:spreadbound} to obtain the bound
\begin{align*}
%\DD{\tQ_{t+1|\tw_{1:t},s}}{\tQ'_{t+1|\tw_{1:t}}}
\Psi_t(\tw_{1:t},s) 
&\le
\frac{\eta^2_t}{2\sigma_{t}^2}\EEcc{\twonorm{g(\tW_t-\xi_t, B_t) - 
g(\tW_t-\xi_t',B_t')}^2}{\tW_{1:t}=\tW_{1:t}'=\tw_{1:t},S=s}.
\end{align*}
Introducing the population gradient $\bg(w) = \EE{g(w,Z)}$, and still using the condition that 
$\tW_t = \tW_t'$, we upper-bound the term in the above expectation as follows:
\begin{align*}
 &\twonorm{g(\tW_t-\xi_t, B_t) - g(\tW_t-\xi_t',B_t')}^2 =
 \twonorm{g(\tW_t-\xi_t, B_t) - \bg(\tW_t) + \bg(\tW_t)- g(\tW_t-\xi_t',B_t')}^2
 \\
 &\qquad\qquad\qquad\qquad\qquad\qquad\le 2\twonorm{g(\tW_t-\xi_t, B_t) - \bg(\tW_t)}^2 + 2\twonorm{\bg(\tW_t)- 
g(\tW_t-\xi_t',B_t')}^2
 \\
 &\qquad\qquad\qquad\qquad\qquad\qquad\le  4\twonorm{g(\tW_t-\xi_t, B_t) - \bg(\tW_t-\xi_t)}^2 + 
4\twonorm{\bg(\tW_t-\xi_t) - \bg(\tW_t)}^2
 \\
 &\quad\qquad\qquad\qquad\qquad\qquad\qquad + 4\twonorm{\bg(\tW_t-\xi'_t)- g(\tW_t-\xi_t',B_t')}^2 + 
\twonorm{\bg(\tW_t-\xi'_t)- \bg(\tW_t)}^2
 \\
 &\qquad\qquad\qquad\qquad\qquad\qquad=  4\twonorm{g(W_t, B_t) - \bg(W_t)}^2 + 4\twonorm{\bg(W_t) - \bg(W_t+\xi_t)}^2
 \\
 &\quad\qquad\qquad\qquad\qquad\qquad\qquad + 4\twonorm{\bg(W_t')- g(W'_t,B_t')}^2 + 4\twonorm{\bg(W_t')- 
\bg(W'_t+\xi_t')}^2,
\end{align*}
where each inequality follows from an application of Cauchy--Schwartz. Plugging the result into 
Equation~\eqref{eq:psibound}, we are left with integrating all terms with respect to the joint distribution of 
$(\tW_{1:t},S)$.

To proceed, note that
\begin{align*}
 &\EEcc{\twonorm{g(W_t, B_t) - \bg(W_t)}^2}{\tW_{1:t}=\tW'_{1:t}=\tw_{1:t},S=s} 
 \\
 &\qquad\qquad\qquad\qquad\qquad\qquad= 
 \EEcc{\twonorm{g(W_t, B_t) - \bg(W_t)}^2}{\tW_{1:t}=\tw_{1:t},S=s}
\end{align*}
due to the independence of $W_t,B_t$ from $\tW_t'$. Thus, we have
\begin{align*}
&\int_s \int_{\tw_{1:t}} \EEcc{\twonorm{g(W_t, B_t) - \bg(W_t)}^2}{\tW_{1:t}=\tW'_{1:t}=\tw_{1:t},S=s} \dd 
\tQ_{1:t|s}(\tw_{1:t}) 
\dd \nu(s) 
\\
 &\int_s \int_{\tw_{1:t}} \EEcc{\twonorm{g(W_t, B_t) - \bg(W_t)}^2}{\tW_{1:t}=\tw_{1:t},S=s} \dd \tQ_{1:t|s}(\tw_{1:t}) 
\dd \nu(s) 
% \\
%  &\quad=\int_s \int_{\tw_{1:t}} \int_{w_t} \EEcc{\twonorm{g(w_t, B_t) - 
% \bg(w_t)}^2}{W_{t}=w_{t},\tW_{1:t}=\tw_{1:t},S=s} \dd 
% Q_{t|\tw_{1:t},s}(w_t) \dd \tQ_{1:t|s}(\tw_{1:t}) 
% \dd \nu(s)
% \\
%  &\quad=\int_s \int_{w_t} \EEcc{\twonorm{g(w_t, B_t) - \bg(w_t)}^2}{W_{1:t}=w_{1:t},S=s} \dd 
% Q_{t|s}(w_t) \dd \nu(s)
\\
&\quad= \EE{\twonorm{g(W_t, B_t) - \bg(W_t)}^2}.
\end{align*}
% where we have used that $\twonorm{g(W_t, B_t) - \bg(W_t)}^2$ is conditionally independent of $\tW_t$ 
% given $W_t$.
Similarly, we observe that 
\begin{align*}
 &\EEcc{\twonorm{g(W_t', B_t') - \bg(W_t')}^2}{\tW_{1:t}=\tw_{1:t},\tW'_{1:t}=\tw_{1:t},S=s} 
 \\
 &\qquad\qquad\qquad\qquad\qquad\qquad= 
 \EEcc{\twonorm{g(W_t', B_t') - \bg(W_t')}^2}{\tW'_{1:t}=\tw_{1:t}}
\end{align*}
due to the independence of $\tW_t',B_t'$ from $\tW_t$ and $S$, so we can write
\begin{align*}
&\int_{s} \int_{\tw_{1:t}} \EEcc{\twonorm{g(W_t', B_t') - 
\bg(W_t')}^2}{\tW_{1:t}=\tW'_{1:t}=\tw_{1:t},S=s} \dd \tQ_{1:t|s}(\tw_{1:t}) \dd \nu(s)
\\
 &\quad=\int_{s} \int_{\tw_{1:t}} \EEcc{\twonorm{g(W_t', B_t') - 
\bg(W_t')}^2}{\tW'_{1:t}=\tw_{1:t}} \dd \tQ_{1:t|s}(\tw_{1:t}) \dd \nu(s)
 \\
 &\quad=\int_{\tw_{1:t}} \EEcc{\twonorm{g(W_t', B_t') - 
\bg(W_t')}^2}{\tW'_{1:t}=\tw_{1:t}} \dd \tQ_{1:t}(\tw_{1:t})
% \\
%  &\quad=\int_{\tw_{1:t}} \int_{w_t'} \EEcc{\twonorm{g(w_t', B_t') - 
% \bg(w_t')}^2}{W'_{t}=w'_{t}} \dd Q_{t|\tw_{1:t}}(w'_t) \dd \tQ_{1:t}(\tw_{1:t})
% \\
%  &\quad= \int_{w_t'} \EEcc{\twonorm{g(w_t', B_t') - \bg(w_t')}^2}{W'_{t}=w'_{t}} \dd Q_{t}(w'_t)
\\
&\quad= \EE{\twonorm{g(W_t', B_t') - \bg(W_t')}^2} = \EE{\twonorm{g(W_t, B_t) - \bg(W_t)}^2},
\end{align*}
where the last step follows from noticing that the marginal distributions of $W_t$ and $W_t'$ are the same. 
% \redd{\textbf{[This last step fails when using Jensen's inequality at the beginning of the proof, as it leads to 
% some nonsensical expressions above that do not correspond to expectations with respect to the true marginal 
% distribution of $W_t$ and $W_t'$.]}}

As for the remaining terms, we first have the following:
\begin{align*}
 &\int_s \int_{\tw_{1:t}} \EEcc{\twonorm{\bg(W_t+\xi_t) - \bg(W_t)}^2}{\tW_{1:t}=\tw_{1:t},S=s} \dd 
\tQ_{1:t|s}(\tw_{1:t}) \dd \nu(s)
\\
 &\quad=\int_{\tw_{t}} \EEcc{\twonorm{\bg(W_t+\xi_t) - \bg(W_t)}^2}{\tW_{t}=\tw_{t}} \dd \tQ_{t}(\tw_{t})
% \\
%  &\quad=\int_{w_t} \int_{\tw_{t}} \EEcc{\twonorm{\bg(W_t+\xi_t) - \bg(W_t)}^2}{\tW_{t}=\tw_{t}} \dd 
% \tQ_{t|w_{t}}(\tw_t) 
% \dd Q_{t}(w_{t})
\\
 &\quad=\int_{w_t} \int_{\tw_{t}} \EEcc{\twonorm{\bg(w_t+\xi_t) - \bg(w_t)}^2}{\tW_{t}=\tw_{t},W_t=w_t} \dd 
\tQ_{t|w_{t}}(\tw_t) \dd Q_{t}(w_{t})
\\
 &\quad=\int_{w_t} \EEcc{\twonorm{\bg(w_t+\xi_t) - \bg(w_t)}^2}{W_t=w_t} \dd Q_{t}(w_{t})
%  \int_{w_t} \Gamma_{\sigma_{1:t}}(w_t) \dd Q_{t}(w_{t}) 
 = \EE{\Gamma_{\sigma_{1:t}}(W_t)},
\end{align*}
where the final step follows from noticing that the conditional distribution of $\xi_t$ given $W_t$ is a zero-mean 
Gaussian with covariance $\sigma^2_{1:t} I$, and recalling the definition of $\Gamma$.
A similar derivation gives
\begin{align*}
 &\int_s \int_{\tw_{1:t}} \EEcc{\twonorm{\bg(W_t'+\xi_t') - \bg(W_t')}^2}{\tW'_{1:t}=\tw_{1:t},S=s} \dd 
\tQ_{1:t|s}(\tw_{1:t}) \dd \nu(s)
\\
 &\quad=\int_{\tw_{t}} \EEcc{\twonorm{\bg(W_t'+\xi_t') - \bg(W_t')}^2}{\tW_{t}'=\tw_{t}} \dd \tQ_{t}(\tw_{t})
% \\
%  &\quad=\int_{w_t'} \int_{\tw_{t}} \EEcc{\twonorm{\bg(W_t'+\xi_t') - \bg(W_t')}^2}{\tW_{t}'=\tw_{t}} \dd 
% Q_{t|w_{t}'}(\tw_t) \dd Q_{t}(w_{t}')
\\
 &\quad=\int_{w_t'} \int_{\tw_{t}} \EEcc{\twonorm{\bg(w_t'+\xi_t') - \bg(w_t')}^2}{\tW_{t}'=\tw_{t},W_t'=w_t'} \dd 
\tQ_{t|w_{t}'}(\tw_t) \dd Q_{t}(w_{t}')
\\
 &\quad=\int_{w_t'} \int_{\tw_{t}} \EEcc{\twonorm{\bg(w_t'+\xi_t') - \bg(w_t')}^2}{W_t'=w_t'} \dd Q_{t}(w_{t}')
 = \EE{\Gamma_{\sigma_{1:t}}(W_t')} = 
\EE{\Gamma_{\sigma_{1:t}}(W_t)},
\end{align*}
where we again used that the marginal distribution of $W_t'$ matches that of $W_t$. The proof is concluded by putting 
everything together.
\end{proof}

\section{Extensions}
In this section, we discuss some additional results that can be derived using the techniques developed in previous 
parts of the paper, as well as propose some open problems for future research.

\subsection{Geometry-aware guarantees}\label{sec:invariant}
One potential criticism regarding the bound of Theorem~\ref{thm:main} is that it heavily depends on the parametrization 
of the loss function. Indeed, measuring the sensitivity of the values and the gradients of the loss function in terms 
of isotropic Gaussian perturbations is somewhat arbitrary, and can result in very conservative bounds. In particular,
the loss function may have better smoothness properties and the gradients could have lower variance when measured in 
terms of different norms, and the final optimum may be more sensitive to perturbations in certain directions than in 
other ones. Luckily, our framework allows addressing these issues by using perturbations of a more general form, 
specifically Gaussian perturbations with general covariance matrices. The key technical result that allows us to take 
advantage of this generalization is the following simple variant of Lemma~\ref{lem:spreadbound}:
\begin{lemma}\label{lem:spreadbound2}
Let $X$ and $Y$ be random variables taking values in $\real^d$ with bounded second moments and let 
$\Sigma$ be an arbitrary symmetric positive definite matrix. Letting 
$\varepsilon\sim\N(0,\Sigma)$ be independent of $X$ and $Y$, the relative entropy between the distributions of $X + 
\varepsilon$ and $Y + \varepsilon$ is 
bounded as
\[
 \DD{P_{X+\varepsilon}}{P_{Y+\varepsilon}} \le \frac 12 \EE{\twonorm{X - Y}_{\Sigma^{-1}}^2}.
\]
\end{lemma}
Accordingly, we can define the generalized local gradient sensitivity and variance functions
\[
\Gamma_{\Sigma,\Sigma'}(w) = \EE{\twonorm{\bg(w) - \bg(w + \xi)}_{\Sigma^{-1}}^2} \quad\mbox{and}\quad 
V_{t,\Sigma}(w) = \EEcc{\twonorm{g(w,B_t) - \bg(w)}_{\Sigma^{-1}}^2}{W_t=w},
 \]
where $\xi \sim \N(0,\Sigma')$, and adapt the definition of $\Delta$ as $\Delta_{\Sigma}(w,s) = \EE{L(w,s) 
- L(w + \xi,s)}$ with $\xi\sim\N(0,\Sigma)$. Then, we can prove the following refined version of Theorem~\ref{thm:main}:
\begin{theorem}\label{thm:main2}
Fix any sequence $\bSigma = \pa{\Sigma_1,\Sigma_2,\dots,\Sigma_T}$ of symmetric positive definite matrices and let 
$\Sigma_{1:t} = \sum_{k=1}^{t-1} \Sigma_k$ for all $t$. Then, the generalization error of the final iterate of SGD 
satisfies
\[
\left|\gen(W_T,S)\right| \le \sqrt{\frac{4R^2}{n} \sum_{t=1}^T \eta_t^2 
\EE{\Gamma_{\Sigma_t,\Sigma_{1:t}}(W_t) 
+ V_{t,\Sigma_t}(W_t)}} + \Bigl|\bEE{\Delta_{\Sigma_{1:T}}(W_T,S') - \Delta_{\Sigma_{1:T}}(W_T,S)}\Bigr|.
\]
\end{theorem}
The proof is left as a straightforward exercise for the reader: one only needs to use the perturbations $\varepsilon_t 
\sim \N(0,\Sigma_t)$ and apply Lemma~\ref{lem:spreadbound2} instead of Lemma~\ref{lem:spreadbound} in the proof of 
Theorem~\ref{thm:main}. As before, this guarantee comes with the attractive property of simultaneously holding for all 
possible choices of $\bSigma$, and is thus able to take advantage of potentially hidden geometric properties of the 
loss landscape.

\subsection{Bounds for general learning algorithms}
Our core idea of conducting a perturbation analysis of the output of SGD can be easily generalized to prove 
generalization guarantees for a much broader family of algorithms. In fact, the following bound can be proved 
without making any assumptions about how the output $W$ is constructed:
\begin{proposition}
Let $W$ and $W'$ be the $\real^d$-dimensional outputs of a learning algorithm $\alg$ run on the independent and 
i.i.d.~data sets $S$ and $S'$. Then, the generalization error of $\alg$ is bounded as
\[
 \left|\gen(W,S)\right| \le \inf_{\Sigma\in\Sw_+} \ev{\sqrt{\frac {R^2}{n} \EE{\norm{W - W'}_{\Sigma^{-1}}^2}} + 
\Bigl|\bEE{\Delta_{\Sigma}(W,S') - \Delta_{\Sigma}(W,S)}\Bigr|}.
\]
\end{proposition}
Notably, this bound only depends on the norms of the parameter vectors and does not involve undesirable 
quantities like Lipschitz constants or the total number of parameters. While we are aware of generalization bounds with 
similar qualities for specific families of deep neural networks (e.g., \citealp{NTS15,BFT17,NBS19,LS19generalization}), 
we 
are not aware of any guarantee of comparable generality in the literature, and we believe that it might be of 
independent interest. We also believe that this bound can be made tighter in special cases by using the chaining 
techniques of \citet{AAV18} and \citet{AA20}.

It is tempting to use the above guarantee to study the generalization properties of SGD. As we show below, this gives 
significantly weaker bounds than our main result. To see this, let us consider the special case $\Sigma = 
\sigma_{1:T}^2 I = \sigma^2T\cdot I$ and write
\begin{align*}
 \frac{1}{2\sigma^2T}\EE{\twonorm{W_T - W_T'}^2} = \frac{1}{2\sigma^2T} \EE{\twonorm{\sum_{t=1}^T \eta_t (G_t 
- G_t')}^2} \le \frac{1}{2\sigma^2} \sum_{t=1}^T \eta_t^2 \EE{\twonorm{G_t - G_t'}^2},
\end{align*}
where the last step follows from Jensen's inequality. While superficially similar to the bound of 
Theorem~\ref{thm:main}, this bound is in fact much weaker since it depends on the \emph{marginal} variance of the 
gradients, which can be very large in general. Nevertheless, assuming that the loss function is $C$-Lipschitz, we can 
simply bound $\norm{G_t - G_t'}\le C$ and obtain a mutual-information bound that is comparable to the 
one proved for SGLD by \citet{PJL18}: for constant $\eta_t = \eta $ and $\sigma_t = \sigma$, both bounds are of order 
$\eta^2 C^2 T/\sigma^2$. Still, as previously discussed, the sensitivity term in our bound can be much larger 
than the excess empirical risk of SGLD, making the guarantee derived using this generic technique weaker overall.

\subsection{Variations on SGD}\label{sec:variations}
It is natural to ask if our techniques are applicable for analyzing other iterative algorithms besides the vanilla SGD 
variant considered in the previous sections. We discuss a few possible extensions below.

\paragraph{Momentum, extrapolation, and iterate averaging.} Following \citet{PJL18}, it is easy to incorporate momentum 
into our updates by simply redefining $w$ as the concatenation of the parameter vector and the rolling geometric 
average 
of the gradients, and redefining the function $g$ appropriately. Gradient extrapolation can be also handled 
similarly---we refer the interested reader to Sections 4.3--4.4 of \citet{PJL18} for details of these extensions. 
An extension not covered in their work is iterate averaging, which can be written as the recursion
\[
 \begin{bmatrix}
 W_{t+1}
 \\
 U_{t+1}
 \end{bmatrix}
 =
 \begin{bmatrix}
 W_{t} - \eta_t g(W_t,Z_{J_t})
 \\
 \gamma_t U_{t} + (1-\gamma_t) W_t
 \end{bmatrix} = \begin{bmatrix}
 I &0
 \\
 (1-\gamma_t)I &\gamma_tI 
 \end{bmatrix}
 \begin{bmatrix}
 W_{t+1}
 \\
 U_{t+1}
 \end{bmatrix}
 - \begin{bmatrix}
 \eta_t g(W_t,Z_{J_t})
 \\
 0
 \end{bmatrix},
\]
where $\gamma_t$ is a sequence of weights in $[0,1]$ (some common choices being $\frac{t-1}{t}$ leading to uniform 
averaging or a constant $\gamma$ that leads to geometric tail-averaging). This SGD variant outputs $U_T = \sum_{t=1}^T 
(1-\gamma_t) \prod_{k=t}^T \gamma_k W_t$. It is easy to deduce 
that our guarantees continue to hold for all these extensions, with the pathwise gradient statistics defined in terms 
of 
the appropriately redefined update function $g$. Thus, our bounds do not show a qualitative improvement in terms of 
generalization error 
for such methods, which may seem counterintuitive since interpolation and iterate averaging are known to offer 
stabilization properties in a variety of settings \citep{NR18,MNR19,lin2020extrapolation}. Not being able to account 
for this effect seems to be an inherent limitation of our technique, most likely caused by upper bounding the mutual 
information between $\tW_T$ and $S$ by the mutual information between the entire SGD path and $S$ in the very first 
step 
of the proof. We do not believe that a simple fix is possible.

\paragraph{Adaptive learning rates and perturbations.}
While our analysis uses a fixed sequence of learning rates $\eta_t$, it may be possible to replace these with adaptive 
stepsizes and even preconditioners such as the ones used in AdaGrad or ADAM \citep{DHS11,KB15}. A straightforward way 
to do this is including these as part of the function $g$, but extra care needs to be taken due to the long-term 
dependence of these stepsizes of the past gradients. We believe that the most common incrementally defined stepsize 
rules can be incorporated in our framework by appropriately conditioning the distribution of $W_t$ on them, but we 
leave 
working out the details of this extension as future work. Similarly, we believe that using adaptive perturbation 
distributions (i.e., choosing each $\sigma_t$ as a function of the history) is possible, but is subject to the same 
challenges. One caveat is that it may be difficult to argue that $\Delta_{\sigma_{1:T}}(W_T,S)$ would be small in this 
case, due to the complicated dependence between the perturbations, $W_T$, and $S$.

\section{Discussion}
While our work has arguably shed some new light on previously unknown aspects of SGD generalization, 
our results are definitely far from being truly satisfactory. Indeed, while the key terms concerning the 
perturbation-sensitivity of the loss function and the variance of the gradients have clear intuitive meaning, it is 
entirely unclear if they are actually the right quantities for characterizing the generalization properties of SGD. In 
fact, we believe that it may be extremely challenging to verify our findings empirically and we find it unlikely 
that matching lower bounds could be shown. Even if the quantities we identify turn out to correctly characterize 
generalization, our results fail to explain \emph{why} running SGD would ever result in trajectories with these 
quantities being small and thus good generalization. For these reasons, we prefer to think of our work as a mere 
first step of a potentially interesting line of research, rather than truly a mature contribution.

One important limitation of our bounds is that they feature a non-standard definition of gradient variance, which can 
make interpretation of the results somewhat difficult. Indeed, while $V_t(w)$ truly corresponds to the variance of the 
stochastic gradient evaluated at $w$ in the single-pass case, its meaning is much less clear when performing multiple 
passes over the data set due to the complicated dependence between the current iterate $W_t$ and the previously sampled 
data points. Thus, in this case, the effect of the choice of hyperparameters like the minibatch size and the number of 
passes is difficult to quantify and requires further investigation.  On a similar note, we remark that the particular 
choice of $\overline{g}(W_t)$ as the population gradient in the definitions of $\Gamma$ and $V$ is not the only 
possible one, and in fact it can be replaced by any $\sigma(W_{1:t})$-measurable function. A natural choice would be 
the time-dependent $\overline{g}_t = \EEcc{g(w,Z)}{W_t=w}$ which would result in a more natural definition of gradient 
variance, but a much less interpretable notion of gradient sensitivity. We leave the exploration of other alternatives 
for future work.

Arguably, the most interesting aspect of our work is the analysis technique we introduced for proving 
Theorem~\ref{thm:mutualbound}. Of course, our proof technique builds on several elements that are familiar from 
previous work. In particular, the core idea of adding noise to the output of learning algorithms to ease analysis has 
been used in numerous works in the past decades. Indeed, early versions of this idea have been proposed by 
\citet{HvC93} and \citet{LC02} from the perspective of PAC-Bayesian generalization bounds 
\citep{McA99,McA13}, which approach has been adapted to modern deep neural networks by \citet{KR17}. Since then, 
PAC-Bayesian bounds have been successfully applied to prove a range of generalization bounds for this important setting 
(see, e.g., \citealp{NBS19,DHGAR21}). More broadly, the idea of adding noise to induce stability (and thus better 
generalization) has also been studied in the literature on differential privacy \citep{D+06,D+06b,CMS11,BST14} and 
adaptive data analysis \citep{D+15,D+15b,FS17,FS18}. That said, we believe that the particular idea of analyzing SGD 
through the noise decomposition in Equation~\eqref{eq:surrogate} is indeed novel, and we expect that this 
idea may have a chance to inspire future analyses of iterative algorithms.

\acks{G.~Neu was supported by ``la Caixa'' Banking Foundation through the Junior Leader Postdoctoral Fellowship 
Programme, a Google Faculty Research Award, and the Bosch AI Young Researcher Award. M.~Haghifam is supported by the 
Vector Institute, Mitacs Accelerate Fellowship, and Ewing Rae Scholarship. D.~M.~Roy was supported, in part, by an 
NSERC Discovery Grant, Ontario Early Researcher Award, and a stipend provided by the Charles Simonyi Endowment.
The first author thanks G\'abor Lugosi for illuminating discussions during the preparation of this work, and an 
anonymous referee who invested serious effort into reviewing the paper and caught a mistake in a previous version of 
the proof of the main theorem. This bug was fixed by using an improved version of Lemma~\ref{lem:spreadbound} suggested 
independently by Tor Lattimore, who the first author also wishes to thank. Finally, the authors are grateful for 
the suggestions of Borja Rodr\'iguez G\'alvez that helped improve the presentation of the results and the rigor of 
several technical details.}

\bibliographystyle{abbrvnat}
\bibliography{ngbib,shortconfs}

\appendix
\section{The proof Lemma~\ref{lem:spreadbound}}\label{sec:spreadbound}
Let us denote the joint distribution of $X,Y$ by $P_{X,Y}$ and observe that 
the respective distributions of $X+\varepsilon$ and $Y+\varepsilon$ can be written as
 \[
  P_{X+\varepsilon} = \int_{x,y} \mathcal{N}(x,\sigma^2 I) \dd P_{X,Y}(x,y) \quad\mbox{and}\quad P_{Y+\varepsilon} = 
\int_{x,y} \mathcal{N}(y,\sigma^2 I) \dd P_{X,Y}(x,y),
 \]
 where $\mathcal{N}(x,\sigma^2I)$ is the Gaussian distribution with mean $x$ and covariance $\sigma^2I$.
Using this observation, we can write
\begin{align*}
 \DD{P_{X+\varepsilon}}{P_{Y+\varepsilon}} &= 
 \DD{\int_{x,y} \mathcal{N}(x,\sigma^2 I) \dd P_{X,Y}(x,y)}{\int_{x,y} \mathcal{N}(y,\sigma^2 
I) \dd P_{X,Y}(x,y)}
\\
&\le 
\int_{x,y} \DD{\mathcal{N}(x,\sigma^2I)}{\mathcal{N}(y,\sigma^2I)} \dd P_{X,Y}(x,y) 
\\
&= \EEs{\DD{\mathcal{N}(X,\sigma^2I)}{\mathcal{N}(Y,\sigma^2I)}}{X,Y} = \frac{1}{2\sigma^2}\EEs{\twonorm{X - Y}^2}{X,Y},
\end{align*}
where the second line uses Jensen's inequality and the joint convexity of $\DD{\cdot}{\cdot}$ in its arguments, and the 
last line follows from noticing that $\DD{\mathcal{N}(x,\Sigma)}{\mathcal{N}(y,\Sigma)} = \frac 12 \norm{x 
- y}_{\Sigma^{-1}}^2$ for any $x,y$ and any symmetric positive definite covariance matrix $\Sigma$.
\qed

\section{The proof of Proposition~\ref{prop:SGLD}}
The proof is very similar to that of Theorem~\ref{thm:mutualbound}, with a few simplifications that are made possible 
by adding the perturbations directly to the iterates.  We start by defining the auxiliary process
\[
 W'_{t+1} = W'_t - \eta_t \bg(W'_t) + \varepsilon_t',
\]
that is similar to the full-batch SGLD process, except in each iteration the gradient update equals to the population 
gradient computed at the current point, i.e., $W'_t$. Also, the perturbation $\varepsilon_t'\sim\N(0,\sigma^2I)$ is 
independent of $W'_t$, and the perturbations are independent across the iterations.

Then, the mutual information between $W_T$ and $S$ can be bounded as
\begin{align*}
 I(W_T;S) &\le \EE{\DD{Q_{T|S}}{Q'_{T}}}
 \le \EE{\DD{Q_{1:T|S}}{Q'_{1:T}}}
 = \sum_{t=1}^T \EE{\DD{Q_{t|W_{1:t-1},S}}{Q'_{t|W_{1:t-1}}}},
\end{align*}
where the first inequality follows from the variational characterization of the mutual information 
\citep{kemperman1974shannon}. Then, the second inequality follows from the data-processing inequality, and the last 
step 
uses the chain rule of the relative entropy. Using the notation introduced above, we rewrite each term as
\begin{align}
 \EE{\DD{Q_{t+1|W_{1:t},S}}{Q'_{t+1|W_{1:t}}}} &= \int_{s} \int_{w_{1:t}} 
\DD{Q_{t+1|w_{1:t},s}}{Q'_{t+1|w_{1:t}}} \dd Q_{1:t|s}(w_{1:t})  \dd \nu(s)\nonumber
\\
&= \int_{s} \int_{w_{1:t}} \Psi_t(w_{1:t},s) \dd Q_{1:t|s}(w_{1:t})  \dd \nu(s),\label{eq:psibound2}
\end{align}
where we defined $\Psi_t(w_{1:t},s)$ as the relative entropy between $W_{t+1}$ and $W_{t+1}'$ 
conditioned on the previous perturbed iterates $W_{1:t} = W_{1:t}' = w_{1:t}$ and the data set $S=s$. It remains 
to bound these terms for all $t$. To do so, we can appeal to Lemma~\ref{lem:spreadbound} to obtain the bound
\begin{align*}
\Psi_t(w_{1:t},s) 
&\le
\frac{\eta^2_t}{2\sigma_{t}^2}\EEcc{\twonorm{g(W_t, B_t) - \bg(W_t)}^2}{W_{1:t}=W'_{1:t}=w_{1:t},S=s}.
\end{align*}
Finally, by plugging this bound into Equation~\eqref{eq:psibound2}, the desired result follows.
\qed

\end{document}